%% file: regularize.tex
\documentclass{article}
\usepackage{colt10e,times,amsfonts,amsmath,amssymb,rotating}

\usepackage{natbib}
\usepackage{algorithm,algorithmic}

\newcommand{\w}{\ensuremath{\mathbf{w}}}
\newcommand{\W}{\ensuremath{\mathbf{W}}}

\newcommand{\M}{\ensuremath{\mathbf{M}}}
\newcommand{\U}{\ensuremath{\mathbf{U}}}
\newcommand{\V}{\ensuremath{\mathbf{V}}}
\newcommand{\bY}{\ensuremath{\mathbf{Y}}}
\newcommand{\bX}{\ensuremath{\mathbf{X}}}

\renewcommand{\v}{\ensuremath{\mathbf{v}}}
\newcommand{\x}{\ensuremath{\mathbf{x}}}
\newcommand{\y}{\ensuremath{\mathbf{y}}}

\newcommand{\half}{\frac{1}{2}}
\newcommand{\thalf}{\tfrac{1}{2}}

\newcommand{\inner}[1]{\left\langle #1 \right\rangle}
\newcommand{\gb}[1]{\boldsymbol{#1}}

\newcommand{\valpha}{\gb{\alpha}}

\newcommand{\D}{\mathcal{D}}

\renewcommand{\eqref}[1]{Eq.~(\ref{#1})}

\newcommand{\secref}[1]{Section~\ref{#1}}
\newcommand{\appref}[1]{Appendix~\ref{#1}}
\newcommand{\thmref}[1]{Theorem.~\ref{#1}}
\newcommand{\lemref}[1]{Lemma.~\ref{#1}}
\newcommand{\corref}[1]{Corollary~\ref{#1}}
\newcommand{\algref}[1]{Algorithm~\ref{#1}}

\newcommand{\X}{\ensuremath{\mathcal{X}}}
\newcommand{\Y}{\ensuremath{\mathcal{Y}}}

\newcommand{\reals}{\ensuremath{\mathbb{R}}}
\newcommand{\Ereals}{\ensuremath{\mathbb{R}^*}}

\newcommand{\E}{\ensuremath{\mathbb{E}}}
\renewcommand{\P}{\ensuremath{\mathbb{P}}}

\newcommand{\Diag}{\mathrm{Diag}}
\newcommand{\Trace}{\mathrm{Tr}}
\newcommand{\sym}{\mathbb{S}}

\newcommand{\sparam}{\beta}

\newcommand{\colnorm}{\Psi}
\newcommand{\rownorm}{\Phi}

\newcommand{\dataset}{\mc{T}}
\newcommand{\mc}{\mathcal}
\def\convexK{\mathcal{K}_c^+}
\def\tx{\tilde{\x}}

\newcommand\zero{\mathbf{0}}
\newcommand\yhat{\hat{y}}

\usepackage{color}

  \renewenvironment{thebibliography}[1]{%
    \begin{oldthebibliography}{#1}%
      \setlength{\parskip}{0ex}%
      \setlength{\itemsep}{0ex}%
  }%
  {%
    \end{oldthebibliography}%
  }

\title{Regularization Techniques for Learning with Matrices}

\author{
Sham M. Kakade\\
The Wharton School\\
University of Pennsylvania\\
\texttt{\small skakade@wharton.upenn.edu}
\And
Shai Shalev-Shwartz\\
School of Computer Sc. \& Eng.\\
The Hebrew University of Jerusalem\\
\texttt{\small shais@cs.huji.ac.il}
\And
Ambuj Tewari\\
Computer Science Department\\
University of Texas at Austin \\
\texttt{\small ambuj@cs.utexas.edu}
}

\begin{document}

\maketitle

\begin{abstract}
  There is growing body of learning problems for which it is natural to
  organize the parameters into matrix, so as to appropriately
  regularize the parameters under some matrix norm (in order to impose
  some more sophisticated prior knowledge). This work describes and
  analyzes a systematic method for constructing such matrix-based, regularization
  methods.  In particular, we focus on how the underlying statistical properties of a given problem can help us decide
  which regularization function is appropriate. 

  Our methodology is based on the known duality fact: that a
  function is strongly convex with respect to some norm if and only if
  its conjugate function is strongly smooth with respect to the dual
  norm. This result has already been found to be a key component in
  deriving and analyzing several learning algorithms.  We demonstrate
  the potential of this framework by deriving novel generalization and
  regret bounds for multi-task learning, multi-class learning, and
  kernel learning.
\end{abstract}

\section{Introduction}

\input{intro}

\section{Preliminaries and Techniques} \label{sec:preliminaries}

In this section we describe the necessary background. Most of the
results below are not new and are based on results from
\cite{Shalev07,KakadeSrTe08,JuditskyNe08}. Nevertheless, we believe
that the presentation given here is simpler and slightly more general.

Our results are based on basic notions from convex analysis and matrix
computation. The reader not familiar with some of the objects
described below may find short explanations in \appref{sec:convex}.

\subsection{Notation}

We consider convex functions $f: \X \to \reals \cup \{\infty\}$,
where $\X$ is a Euclidean vector space equipped with an inner product
$\inner{\cdot, \cdot}$. We denote $\Ereals = \reals \cup \{\infty\}$.
The subdifferential of $f$ at $x \in \X$ is denoted by $\partial
f(x)$. The Fenchel conjugate of $f$  is denoted by $f^\star$. Given a
norm $\|\cdot\|$, its dual norm is denoted by $\|\cdot\|_\star$. 
We say that a convex function is $V$-Lipschitz w.r.t. a norm
$\|\cdot\|$ if for all $x \in \X$ exists $v \in \partial f(x)$ with
$\|v\| \le V$. Of particular interest are $p$-norms,
$\|x\|_p=(\sum_i |x_i|^p)^{1/p}$. 

When dealing with matrices, We consider the vector space $\X=\reals^{m
  \times n}$ of real matrices of size $m \times n$ and the vector
space $\X=\sym^n$ of symmetric matrices of size $n \times n$, both
equipped with the inner product, $ \inner{\bX,\bY} := \Trace(\bX^\top
\bY) $.  Given a matrix $\X$, the vector $\sigma(\bX)$ is the vector
that contains the singular values of $\bX$ in a non-increasing
order. For $\bX \in \sym^n$, the vector $\lambda(\bX)$ is the vector
that contains the eigenvalues of $\bX$ arranged in non-increasing
order.

\subsection{Strong Convexity--Strong Smoothness Duality}

Recall that the domain of $f: \X \to \Ereals$
is $\{ x \::\: f(x) < \infty\}$ (allowing $f$ to take infinite values is the effective way to restrict its domain to a proper subset of $\X$). We first define strong convexity.

\begin{definition}
A function $f : \X \to \Ereals$ is $\sparam$-strongly convex
w.r.t. a norm $\|\cdot\|$ if for all $x,y$ in the relative interior
of the domain of $f$ and $\alpha \in (0,1)$ we have
\begin{equation*}
f(\alpha x + (1-\alpha) y) 
\le \alpha f(x) + (1-\alpha)f(y) 
- \thalf \sparam \alpha (1-\alpha) \|x-y\|^2
\end{equation*}
\end{definition}

We now define strong smoothness. Note that a strongly smooth function $f$ is always finite.

\begin{definition}
A function $f : \X \to \reals$ is $\sparam$-strongly smooth
w.r.t. a norm $\|\cdot\|$ if $f$ is everywhere differentiable and if for all $x,y$  we have
$$
f(x+y) \leq f(x) + \inner{\nabla f(x),y}
+  \thalf \beta \|y\|^2
$$
\end{definition}

The following theorem
states that strong convexity and strong smoothness are dual
properties.  Recall that the biconjugate $f^{\star \star}$ equals $f$ if and only if $f$ is
closed and convex. 

\begin{theorem} \label{thm:strongsmooth}
(Strong/Smooth Duality) Assume that $f$ is a closed and convex function. Then $f$ is $\sparam$-strongly convex w.r.t. a norm $\|\cdot\|$ if
and only if $f^\star$ is $\tfrac{1}{\sparam}$-strongly smooth w.r.t.
the dual norm $\|\cdot\|_\star$.
\end{theorem}

Subtly, note that while the domain of a strongly convex function $f$
may be a proper subset of $\X$ (important for a number of settings),
its conjugate $f^\star$ always has a domain which is $\X$ (since if $f^\star$ is strongly smooth then it is finite and everywhere differentiable).
The above theorem can be found, for instance, in~\cite{Zalinescu02} (see Corollary 3.5.11 on p. 217 and Remark 3.5.3 on p. 218). In the machine learning literature, a proof of one direction (strong convexity $\Rightarrow$ strong smoothness) can be found in \cite{Shalev07}. We could not find a proof of the reverse implication in a place easily accessible to machine learning people. So, a self-contained proof is provided in the appendix. 

The following direct corollary of \thmref{thm:strongsmooth} is central
in proving both regret and generalization bounds.

\begin{corollary} \label{cor:important}
If $f$ is $\sparam$ strongly convex w.r.t. $\|\cdot\|$ and $f^\star(\mathbf{0})
= 0$, then, denoting the partial sum $\sum_{j\le i} v_j$ by $v_{1:i}$, we have, for any sequence 
$v_1,\ldots,v_n$ and for any $u$,
\begin{equation*}
\sum_{i=1}^n  \inner{v_i,u} -f (u) \le
f^\star(v_{1:n}) 
\le~ \sum_{i=1}^n \inner{\nabla f^\star(v_{1:{i-1}}),v_i}
+  \frac{1}{2\sparam} \sum_{i=1}^n \|v_i\|_\star^2 \ . 
\end{equation*}
\end{corollary}
\begin{proof}
The 1st inequality is Fenchel-Young and the 2nd
is from the definition of smoothness by induction.
\end{proof}


\subsection{Machine learning implications of the strong-convexity / strong-smoothness 
duality }

We consider two learning models.
\begin{itemize}
\item \textbf{Online convex optimization:}
Let $\mc{W}$ be a convex set. Online convex optimization is a two player repeated game. On round $t$ of the game, the learner (first player) should choose $w_t \in \mc{W}$ and the environment (second player) responds with a convex function over $\mc{W}$, i.e. $l_t : \mc{W} \to \reals$. The goal of the learner is to minimize its regret defined as:
\[
\frac{1}{n} \sum_{t=1}^n l_t(w_t) - \min_{w \in \mc{W}}  \frac{1}{n} \sum_{t=1}^n l_t(w) ~.
\]
\item \textbf{Batch learning of linear predictors:}
Let $\D$ be a distribution over $\X \times \Y$. Our goal is to learn a prediction rule from $\X$ to $\Y$. The prediction rule we use is based on a linear mapping $x \mapsto \inner{w,x}$, and the quality of the prediction is assessed by a loss function $l(\inner{w,x},y)$. Our primary goal is to find $w$ that has low risk (a.k.a. generalization error), defined as $L(w) = \E[l(\inner{w,x},y)]$, where expectation is with respect to $\D$. To do so, we can sample $n$ i.i.d. examples from $\D$ and observe the empirical risk, $\hat{L}(w) = \frac{1}{n} \sum_{i=1}^n l(\inner{w,x_i},y_i)$. The goal of the learner is to find $\hat{w}$ with a low excess risk defined as:
\[
L(\hat{w}) - \min_{w \in \mc{W}} L(w) ~,
\]
where $\mc{W}$ is a set of vectors that forms the comparison class.

\end{itemize}

We now seamlessly provide learning guarantees for both models based on
\corref{cor:important}. We start with the online convex optimization
model. 

\paragraph{Regret Bound for Online Convex Optimization}
\algref{alg:frl} provides one common algorithm 
which achieves the following regret bound. It is one of a family of algorithms that enjoy the same regret bound (see~\cite{Shalev07}).

\begin{theorem} \label{thm:regretbound}
(Regret) Suppose \algref{alg:frl} is used with a function $f$ that is $\sparam$-strongly
convex w.r.t. a norm $\|\cdot\|$ on $\mc{W}$ and has $f^\star(\mathbf{0}) = 0$.
Suppose the loss functions $l_t$ are
convex and $V$-Lipschitz w.r.t. the dual norm $\|\cdot\|_\star$. Then, the algorithm
run with any positive $\eta$ enjoys the regret bound,
$$
	\sum_{t=1}^T l_t(w_t) - \min_{u \in \mc{W}} \sum_{t=1}^T l_t(u)
	\le \frac{\max_{u \in \mc{W}} f(u)}{\eta} + \frac{\eta V^2 T}{2\sparam} 
$$
\end{theorem}
\begin{proof}
Apply \corref{cor:important} to the sequence $-\eta v_1,\ldots,-\eta v_T$ to get, for all $u$,
$$
	-\eta \sum_{t=1}^T \inner{v_t, u} - f(u) \le -\eta \sum_{t=1}^T \inner{v_t, w_t} +
	\frac{1}{2\sparam} \sum_{t=1}^T \|\eta v_t\|_\star^2\ .
$$
Using the fact that $l_t$ is $V$-Lipschitz, we get $\|v_t\|_\star \le V$.
Plugging this into the inequality above and rearranging gives,
$
	\sum_{t=1}^T \inner{v_t,w_t - u} \le \frac{f(u)}{\eta} +
	\frac{\eta V^2 T}{2\sparam}
$.
By convexity of $l_t$, $l_t(w_t) - l_t(u) \le \inner{v_t,w_t-u}$. Therefore,
$
	\sum_{t=1}^T l_t(w_t) - \sum_{t=1}^T l_t(u) \le \frac{f(u)}{\eta} +
	\frac{\eta V^2 T}{2\sparam}
$.
Since the above holds for all $u \in \mc{W}$ the result follows. 
\end{proof}

\begin{algorithm}[top]
\caption{Online Mirror Descent}
\label{alg:frl}
\begin{algorithmic}
\STATE $w_1 \gets \nabla f^\star(\mathbf{0})$
\FOR{$t = 1$ to $T$}
	\STATE Play $w_t \in \mc{W}$
	\STATE Receive $l_t$ and pick $v_t \in \partial l_t(w_t)$
	\STATE $w_{t+1} \gets \nabla f^\star\left( -\eta \sum_{s=1}^t v_t\right)$
\ENDFOR
\end{algorithmic}
\end{algorithm}

\paragraph{Generalization bound for the batch model via Rademacher analysis}
Let $\dataset = ((x_1,y_1),\ldots,(x_n,y_n)) \in (\X \times \Y)^n$ be a training set
obtained by sampling i.i.d. examples from $\D$.  For a class of real valued
functions $\mc{F}\subseteq \reals^\mc{X}$, define its Rademacher
complexity on $\dataset$ to be
$$
	\mc{R}_\dataset(\mc{F}) :=  \E \left[ \sup_{f \in \mc{F}} \frac{1}{n}\sum_{i=1}^n
	\epsilon_i f(x_i) \right]\ .
$$
Here, the expectation is over $\epsilon_i$'s, which are
i.i.d. Rademacher random variables, i.e. $\P(\epsilon_i = -1) =
\P(\epsilon_1 = +1) = \half$.  It is well known that bounds on
Rademacher complexity of a class immediately yield generalization
bounds for classifiers picked from that class (assuming the loss
function is Lipschitz). Recently, \cite{KakadeSrTe08} proved
Rademacher complexity bounds for classes consisting of linear
predictors using strong convexity arguments. We now give a quick proof
of their main result using \corref{cor:important}. This proof is
essentially the same as their original proof but highlights the
importance of \corref{cor:important}.

\begin{theorem}
\label{thm:rademacharbound}
(Generalization) Let $f$ be a $\sparam$-strongly convex function
w.r.t. a norm $\|\cdot\|$ and assume that $f^\star(\mathbf{0}) = 0$. Let
$\mc{X} = \{ x \::\: \|x\|_\star \le X \}$ and $\mc{W} = \{ w \::\:
f(w) \le f_\mathrm{max} \}$. Consider the class of linear functions,
$
	\mc{F} = \{ x \mapsto \inner{w,x} \::\: w \in \mc{W} \}
$.
Then, for any dataset $\dataset \in \mc{X}^n$, we have
$$
	\mc{R}_\dataset(\mc{F}) \le X\sqrt{\frac{2f_\mathrm{max}}{\sparam n}}\ .
$$
\end{theorem}
\begin{proof}
Let $\lambda > 0$.
Apply \corref{cor:important} with $u = w$ and $v_i = \lambda \epsilon_i x_i$ to get,
\begin{align*}
\sup_{w \in \mc{W}} \sum_{i=1}^n \inner{w,\lambda \epsilon_i x_i}
&\le \frac{\lambda^2}{2\sparam} \sum_{i=1}^n \|\epsilon_i x_i\|_\star^2
+ \sup_{w \in \mc{W}} f(w) + \sum_{i=1}^n \inner{ \nabla f^\star(v_{1:i-1}), \epsilon_i x_i} \\
&\le \frac{\lambda^2 X^2 n}{2\sparam}
+ f_{\mathrm{max}} + \sum_{i=1}^n \inner{ \nabla f^\star(v_{1:i-1}), \epsilon_i x_i} \ . 
\end{align*}
Now take expectation on both sides. The left hand side is $n\lambda \mc{R}_\dataset(\mc{F})$
and the last term on the right hand side becomes zero. Dividing throughout by $n\lambda$, we get,
$
	\mc{R}_\dataset(\mc{F}) \le \frac{\lambda X^2}{2\sparam} +
	\frac{f_{\mathrm{max}}} {n \lambda}\
$.
Optimizing over $\lambda$ gives us the
result.
\end{proof}

Combining the above with the contraction lemma and standard Rademacher
based generalization bounds (see e.g. \cite{BartlettMe02,KakadeSrTe08}) we obtain:
\begin{corollary} \label{cor:rademacharbound}
 Let $f$ be a $\sparam$-strongly convex function
w.r.t. a norm $\|\cdot\|$ and assume that $f^\star(\mathbf{0}) = 0$. Let
$\mc{X} = \{ x \::\: \|x\|_\star \le X \}$ and $\mc{W} = \{ w \::\:
f(w) \le f_\mathrm{max} \}$. 
Let $l$ be an $\rho$-Lipschitz scalar loss function and let $\D$ be an
arbitrary distribution over $\X \times \Y$. Then, the algorithm that
receives $n$ i.i.d. examples and returns $\hat{w}$ that minimizes the
empirical risk, $\hat{L}(w)$, satisfies
\[
\E\left[ L(\hat{w}) - \min_{w \in \mc{W}} L(w) \right] ~\le~ 
O\left(\rho\,X\sqrt{\frac{f_\mathrm{max}}{\sparam n}} \right)\ ,
\]
where expectation is with respect to the choice of the $n$ i.i.d. examples.
\end{corollary}
We note that it is also easy to obtain a generalization bound that
holds with high probability, but for simplicity of the presentation we
stick to expectations. 

\subsection{Strongly Convex Matrix Functions}

Before we consider strongly convex matrix functions, let us recall the following 
result about strong convexity of vector $\ell_p$ norm.
Its proof can be found e.g. in \cite{Shalev07}.

\begin{lemma}\label{lem:vectorlq}
Let $q \in [1,2]$. The function $f : \reals^d \to \reals$ defined as
$f(w) = \half \|w\|_q^2$ is $(q-1)$-strongly convex with respect to
$\|\cdot\|_q$ over $\reals^d$. 
\end{lemma}

We mainly use the above lemma to obtain results with respect to
the norms $\|\cdot\|_2$ and $\|\cdot\|_1$. The case $q=2$ is
straightforward.  Obtaining results with respect to $\|\cdot\|_1$ is
slightly more tricky since for $q=1$ the strong convexity parameter is
$0$ (meaning that the function is not strongly convex). To overcome
this problem, we shall set $q$ to be slightly more than $1$, e.g. $q =
\tfrac{\ln(d)}{\ln(d)-1}$. For this choice of $q$, the strong
convexity parameter becomes
$q-1 = 1/(\ln(d)-1) \ge 1/\ln(d)$
and the value of $p$ corresponds
to the dual norm is $p = \left(1-1/q\right)^{-1} = \ln(d) $. 
Note that for any $x \in \reals^d$ we have
\[
\|x\|_\infty \le \|x\|_p \le (d \|x\|_\infty^p)^{1/p} = d^{1/p}
\|x\|_\infty = e\,\|x\|_\infty \le 3\,\|x\|_\infty ~.
\]
Hence the dual norms are also equivalent up to a factor of $3$:
$\|w\|_1 \ge \|w\|_q \ge \|w\|_1/3$.
The above lemma therefore implies the following corollary.
\begin{corollary}
\label{cor:vectorl1}
The function $f: \reals^d \to \reals$ defined as $f(w) = \half \|w\|_q^2$
for $q = \tfrac{\ln(d)}{\ln(d)-1}$ is $1/(3 \ln(d))$-strongly convex
with respect to $\|\cdot\|_1$ over $\reals^d$.
\end{corollary}

We now consider two families of strongly convex matrix functions.

\paragraph{Schatten $q$-norms}
The first result we need is the counterpart of \lemref{lem:vectorlq} for the $q$-Schatten norm defined as $\|\bX\|_{S(q)} := \|\sigma(\bX)\|_q$
This result can be found in~\cite{BallCaLi94}.

\begin{theorem} \label{thm:ui_examples} (Schatten matrix functions)
Let $q\in [1,2]$. The function $F:\reals^{m \times n} \to \reals$ defined as
$F(\bX) = \half \|\sigma(\bX)\|_q^2$ is $(q-1)$-strongly convex w.r.t. the
$q$-Schatten norm $\|\bX\|_{S(q)} := \|\sigma(\bX)\|_q$ over $\reals^{m \times n}$.
\end{theorem}

As above, choosing $q$ to be $\tfrac{\ln m'}{\ln(m') - 1}$ for $m' = \min\{m,n\}$ gives the following corollary.

\begin{corollary}
\label{schattenl1}
The function $F: \reals^{m \times n} \to \reals$ defined as $F(\W) = \half \|\W\|_{S(q)}^2$
for $q = \tfrac{\ln(m')}{\ln(m')-1}$ is $1/(3 \ln(m'))$-strongly convex
with respect to $\|\cdot\|_{S(1)}$ over $\reals^{m \times n}$.
\end{corollary}

\paragraph{Group Norms.}

Let $\bX = (\bX^1 \bX^2 \ldots \bX^n)$ be a $m \times n$ real matrix with {\em columns}
$\bX^i\in \reals^m$. We denote by $\|\bX\|_{r,p}$ as
$$
	\|\bX\|_{r,p} :=
\|\,(\|\bX^1\|_r,\ldots,\|\bX^n\|_r)\,\|_p\ .
$$
That is, we apply $\|\cdot\|_r$ to each column of $\bX$ to get a
vector in $\reals^n$ to which we apply the norm $\|\cdot\|_p$ to get
the value of $\|\bX\|_{r,p}$. It is easy to check that this is indeed
a norm.  The dual of $\|\cdot\|_{r,p}$ is $\|\cdot\|_{s,t}$ where
$1/r+1/s=1$ and $1/p+1/t=1$.  The following theorem, which appears in
a slightly weaker form in~\cite{JuditskyNe08}, provides us with an
easy way to construct strongly convex group norms.  We provide a proof
in the appendix which is much simpler than that of \cite{JuditskyNe08}
and is completely ``calculus free''.

\begin{theorem}
\label{thm:grpsmoothness}
(Group Norms) Let $\colnorm, \rownorm$ be absolutely symmetric norms on $\reals^m, \reals^n$. 
Let $\rownorm^2 \circ \sqrt{}:\reals^n \to \Ereals$ denote the following function,
\begin{equation}
\label{eqn:quadcond}
	(\rownorm^2 \circ \sqrt{})(x) := \rownorm^2(\sqrt{x_1},\ldots,\sqrt{x_n})\ .
\end{equation}
Suppose, $(\rownorm^2 \circ \sqrt{})$ is a norm on $\reals^n$. Further, let the functions
$\colnorm^2$ and $\rownorm^2$ be $\sigma_1$- and $\sigma_2$-smooth w.r.t.
$\colnorm$ and $\rownorm$ respectively. Then, $\|\cdot\|^2_{\colnorm,\rownorm}$
is $(\sigma_1 + \sigma_2)$-smooth w.r.t. $\|\cdot\|_{\colnorm,\rownorm}$.
\end{theorem}

The condition that~\eqref{eqn:quadcond} be a norm appears strange but in fact it already occurs in the literature. Norms satisfying it
are called {\em quadratic symmetric gauge} functions (or Q-norms)
\cite[p. 89]{Bhatia97}. It is easy to see that $\|\cdot\|_p$ for $p \ge 2$ is a Q-norm. Now using strong convexity/strong smoothness duality
and the discussion preceding \corref{cor:vectorl1}, we get the following corollary.

\begin{corollary}
\label{cor:groupl1}
The function $F: \reals^{m \times n} \to \reals$ defined as $F(\W) = \half \|\W\|_{2,q}^2$
for $q = \tfrac{\ln(n)}{\ln(n)-1}$ is $1/(3 \ln(n))$-strongly convex
with respect to $\|\cdot\|_{2,1}$ over $\reals^{m \times n}$.
\end{corollary}

\subsection{Putting it all together} \label{sec:close1}

Combining \lemref{lem:vectorlq} and \corref{cor:vectorl1} with the bounds given in
\thmref{thm:regretbound} and \corref{cor:rademacharbound} we therefore obtain the following two corollaries. 
\begin{corollary}
  Let $\mc{W} = \{w : \|w\|_1 \le W\}$ and let $l_1,\ldots,l_n$ be a
  sequence of functions which are $X$-Lipschitz
  w.r.t. $\|\cdot\|_\infty$. Then, there exists an online algorithm
  with a regret bound of the form
\[
\frac{1}{n} \sum_{t=1}^n l_t(w_t) - \min_{w \in \mc{W}}  \frac{1}{n}
\sum_{t=1}^n l_t(w) 
~\le~ O\left(
X\,W\,\sqrt{\frac{\ln(d)}{n}}
\right) ~.
\]
\end{corollary}

\begin{corollary}
  Let $\mc{W} = \{w : \|w\|_1 \le W\}$ and let $\X = \{x \in \reals^d
  : \|x\|_\infty \le X\}$. Let $l$ be an $\rho$-Lipschitz scalar loss
  function and let $\D$ be an arbitrary distribution over $\X \times
  \Y$. Then, there exists a batch learning algorithm that returns a
  vector $\hat{w}$ such that
\[
\E\left[  L(\hat{w}) -  \min_{w \in \mc{W}} L(w)\right] ~\le~
O\left(
X\,W\,\sqrt{\frac{\ln(d)}{n}}
\right) ~.
\]
\end{corollary}

Results of the same flavor can be obtained for learning matrices. For
simplicity, we present the following two corollaries only for the
online model, but it is easy to derive their batch
counterparts. 

\begin{corollary} \label{cor:online21}
  Let $\mc{W} = \{\W \in \reals^{k \times d} : \|\W\|_{2,1} \le W\}$ and let $l_1,\ldots,l_n$ be a
  sequence of functions which are $X$-Lipschitz
  w.r.t. $\|\cdot\|_{2,\infty}$. Then, there exists an online algorithm
  with a regret bound of the form
\[
\frac{1}{n} \sum_{t=1}^n l_t(\W_t) - \min_{\W \in \mc{W}}  \frac{1}{n}
\sum_{t=1}^n l_t(\W) 
~\le~ O\left(
X\,W\,\sqrt{\frac{\ln(d)}{n}}
\right) ~.
\]
\end{corollary}

\begin{corollary}
  Let $\mc{W} = \{\W \in \reals^{k \times d} : \|\W\|_{S(1)} \le W\}$ and let $l_1,\ldots,l_n$ be a
  sequence of functions which are $X$-Lipschitz
  w.r.t. $\|\cdot\|_{S(\infty)}$. Then, there exists an online algorithm
  with a regret bound of the form
\[
\frac{1}{n} \sum_{t=1}^n l_t(\W_t) - \min_{\W \in \mc{W}}  \frac{1}{n}
\sum_{t=1}^n l_t(\W) 
~\le~ O\left(
X\,W\,\sqrt{\frac{\ln(\min\{k,d\})}{n}}
\right) ~.
\]
\end{corollary}

\section{Matrix Regularization} \label{sec:reg}

We are now ready to demonstrate the power of the general techniques we
derived in the previous section. Consider a learning problem (either
online or batch) in which $\X$ is a subset of a matrix space (of
dimension $k \times d$) and we would like to
learn a linear predictor of the form $\bX \mapsto \inner{\W,\bX}$ where
$\W$ is also a matrix of the same dimension. The loss function takes
the form $l(\inner{\W,\bX},y)$ and we assume for simplicity that $l$
is $1$-Lipschitz with respect to its first argument. For example,
$l$ can be the absolute loss, $l(a,y) = |a-y|$,  or the hinge-loss, $l(a,y) = \max\{0,1-ya\}$. 

For the sake of concreteness, let us focus on the batch learning
setting, but we note that the discussion below is relevant to the
online learning model as well. Our prior knowledge on the learning
problem is encoded by the definition of the comparison class $\mc{W}$
that we use. In particular, all the comparison classes we use
take the form $\mc{W} = \{\W : \|\W\| \le W\}$, where the only
difference is what norm do we use. We shall compare the following
four classes:
\begin{align*}
\mc{W}_{1,1} &= \{\W : \|\W\|_{1,1} \le W_{1,1}\}
&
\mc{W}_{2,2} &= \{\W : \|\W\|_{2,2} \le W_{2,2}\}
\\
\mc{W}_{2,1} &= \{\W : \|\W\|_{2,1} \le W_{2,1}\}
&
\mc{W}_{S(1)} &= \{\W : \|\W\|_{S(1)} \le W_{S(1)}\}
\end{align*}

Let us denote $X_{\infty,\infty} = \sup_{x \in \X}
\|\bX\|_{\infty,\infty}$. We define $X_{2,2},X_{2,\infty},X_{S(\infty)}$ analogously.  Applying the
results of the previous section to these classes we obtain the
bounds given in Table \ref{tab:bounds} where for simplicity we ignore constants.
\begin{table}[h] 
\begin{center}
\begin{tabular}{|l|c|c|c|c|}
\hline
& & & & \\[-0.1cm]
class & $\mc{W}_{1,1}$ & $\mc{W}_{2,2}$ & $\mc{W}_{2,1}$ &
$\mc{W}_{S(1)}$ \\[0.1cm] \hline
& & & & \\ 
bound &  $W_{1,1}\,X_{\infty,\infty} \sqrt{\tfrac{\ln(kd)}{n}}$ &
$W_{2,2}\,X_{2,2}\,\sqrt{\tfrac{1}{n}}$ &
$W_{2,1}\,X_{2,\infty}\,\sqrt{\tfrac{\ln(d)}{n}}$ &
$W_{S(1)}\,X_{S(\infty)}\,\sqrt{\tfrac{\ln(\min\{d,k\})}{n}}$ \\[0.2cm]\hline
\end{tabular}
\end{center}
\caption{List of bounds for learning with matrices. For simplicity we
  ignore constants.} \label{tab:bounds}
\end{table}

Let us now discuss which class should be used based on prior knowledge
on properties of the learning problem. We start with the well known
difference between $\mc{W}_{1,1}$ and $\mc{W}_{2,2}$. Note that both
of these classes ignore the fact that $\W$ is organized as a $k \times
d$ matrix and simply refer to $\W$ as a single vector of dimension
$kd$.  The difference between $\mc{W}_{1,1}$ and $\mc{W}_{2,2}$ is
therefore the usual difference between $\ell_1$ and $\ell_2$
regularization. To understand this difference, suppose that $\W$
is some matrix that performs well on the distribution we have. Then,
we should take the radius of each class to be the minimal possible
while still containing $\W$, namely, either $\|\W\|_{1,1}$
or $\|\W\|_{2,2}$. Clearly, $\|\W\|_{2,2} \le
\|\W\|_{1,1}$ and therefore in terms of this term there is a
clear advantage to use the class $\mc{W}_{2,2}$. On the other hand,
$X_{2,2} \ge X_{\infty,\infty}$. We therefore need to understand which
of these inequalities is more important. Of course, in general, the
answer to this question is data dependent. However, we can isolate
properties of the distribution that can help us choose the better
class.

One useful property is sparsity of either $\bX$ or $\W$. If $\bX$ is
assumed to be $s$ sparse (i.e., it has at most $s$ non-zero elements), 
then we have $X_{2,2} \le \sqrt{s} X_{\infty,\infty}$. That is, for a
small $s$, the difference between $X_{2,2}$ and $X_{\infty,\infty}$ is
small. In contrast, if $\bX$ is very dense and each of its entries is
bounded away from zero, e.g. $\X \in \{\pm 1\}^{k \times d}$, then
$\|\bX\|_{2,2} = \sqrt{kd} \|\bX\|_{\infty,\infty}$. The same arguments
are true for $\W$. Hence, with prior knowledge about the sparsity of
$\bX$ and $\W$ we can guess which of the bounds will be smaller.

Next, we tackle the more interesting cases of $\mc{W}_{2,1}$ and
$\mc{W}_{S(1)}$. For the former, recall that we first apply $\ell_2$
norm on each column of $\W$ and then apply $\ell_1$ norm on the obtained
vector of norm values. Similarly, to calculate $\|\bX\|_{2,\infty}$ we
first apply $\ell_2$ norm on columns of $\bX$ and then apply
$\ell_\infty$ norm on the obtained vector of norm values. 
Let us now compare $\mc{W}_{2,1}$  to $\mc{W}_{1,1}$. Suppose that the
columns of $\bX$ are very sparse. Therefore, the $\ell_2$ norm of
each column of $\bX$ is very close to its $\ell_\infty$ norm. 
On the other hand, if some of the columns of $\W$ are dense, then
$\|\W\|_{2,1}$ can be order of $\sqrt{k}$ smaller than $\|\W\|_{1,1}$.
In that case, the class $\mc{W}_{2,1}$ is preferable over the class $\mc{W}_{1,1}$.
As we show later, this is the case in multi-class problems, and
we shall indeed present an improved multi-class algorithm that uses
the class $\mc{W}_{2,1}$. Of course, in some problems, columns of
$\bX$ might be very dense while columns of $\W$ can be sparse. In such
cases, using $\mc{W}_{1,1}$ is better than using $\mc{W}_{2,1}$. 

Now lets compare $\mc{W}_{2,1}$ to $\mc{W}_{2,2}$. Similarly to the
previous discussion, choosing $\mc{W}_{2,1}$ over $\mc{W}_{2,2}$ makes
sense if we assume that the vector of $\ell_2$ norms of columns,
$(\|\W^1\|_2,\ldots,\|\W^d\|_2)$, is sparse. This implies that we
assume a ``group''-sparsity pattern of $\W$, i.e., each column of $\W$
is either the all zeros column or is dense. This type of
grouped-sparsity has been studied in the context of group Lasso and
multi-task learning.  Indeed, we present bounds for multi-task
learning that relies on this assumption.  Without the group-sparsity
assumption, it might be better to use $\mc{W}_{2,2}$ over
$\mc{W}_{2,1}$.

Finally, we discuss when it makes sense to use $\mc{W}_{S(1)}$. Recall
that $\|\W\|_{S(1)} = \|\sigma(\W)\|_1$, where $\sigma(\W)$ is the
vector of singular values of $\W$, and $\|\bX\|_{S(\infty)} =
\|\sigma(\bX)\|_\infty$. Therefore, the class $\mc{W}_{S(1)}$ should be
used when we assume that the \emph{spectrum} of $\W$ is sparse while
the spectrum of $\bX$ is dense. This means that the prior knowledge we
employ is that $\W$ is of low rank while $\X$ is of high rank.  Note
that $\mc{W}_{2,2}$ can be defined equivalently as
$\mc{W}_{S(2)}$. Therefore, the difference between $\mc{W}_{S(1)}$ and
$\mc{W}_{2,2}$ is similar to the difference between $\mc{W}_{1,1}$ and
$\mc{W}_{2,2}$ just that instead of considering sparsity properties of
the elements of $\W$ and $\bX$ we consider sparsity properties of the
spectrum of $\W$ and $\bX$.

In the next sections we demonstrate
how to apply the general methodology described above in
order to derive a few generalization and regret bounds for problems
of recent interest. 

\section{Multi-task learning} \label{sec:multi-task}

Suppose we are simultaneously solving $k$-multivariate prediction
problems, where each learning example is of the form $(\bX,\y)$ where
$\bX \in \reals^{k \times d}$ is a matrix of example vectors with
examples from different tasks sitting in rows of $\bX$, and $\y \in
\reals^k$ are the responses for the $k$ problems. To predict the $k$
responses, we learn a matrix $\W \in \reals^{k \times d}$ such that
$\Diag(\W^\top \bX)$ is a good predictor of $\y$. In this section, we
denote {\em row} $j$ of $\W$ by $\w^j$. The predictor for the $j$th
task is therefore $\w^j$. The quality of a prediction
$\inner{\w^j,\x^j}$ for the $j$'th task is assessed by a loss function
$l^j : \reals \times \Y^j \to \reals$; And, the total loss of
$\W$ on an example $(\bX,\y)$ is defined to be the sum of the
individual losses,
\[
l(\W,\bX,\y) = \sum_{j=1}^k l^j(\inner{\w^j,\x^j},y^j) ~.
\]
This formulation allows us to mix regression and classification
problems and even use different loss functions for different
tasks. Such ``heterogeneous'' multi-task learning has attracted recent
attention \citep{YangKiXi09}.

If the tasks are related, then it is natural to use regularizers that
``couple'' the tasks together so that similarities across tasks can be
exploited. Considerations of common sparsity patterns (same features
relevant across different tasks) lead to the use of group norm
regularizers (i.e. using the comparison class $\mc{W}_{2,1}$ defined
in the previous section) while rank considerations (the $\w^j$'s lie
in a low dimensional linear space) lead to the use of unitarily
invariant norms as regularizers (i.e. the comparison class is
$\mc{W}_{S(1)}$).

We now describe online and batch multi-task learning using different
matrix norm.

\subsection{Online multi-task learning}
In the online model, on round $t$ the learner first uses $\W_t$ to
predict the vector of responses and then it pays the cost 
$l_t(\W_t)
= l(\W_t,\bX_t,\y_t) = \sum_{j=1}^k l^j\left(\inner{\w_t^j,\x_t^j},y_t^j\right) $.
Let $\V_t \in \reals^{k\times d}$ be a sub-gradient of $l_t$ at $\W_t$. It
is easy to verify that the $j$'th row of $\V_t$, denoted $\v_t^j$, is a sub-gradient of $
l^j\left(\inner{\w_t^j,\x_t^j},y_t^j\right)$ at $\w_t^j$. Assuming that $l^j$ is
$\rho$-Lipschitz   with respect to its first argument, we obtain that
$\v_t^j = \tau_t^j \x_t^j$ for some $\tau_t^j \in [-\rho,\rho]$. In other
words, $\V_t = \Diag(\tau_t)\,\bX_t$. It is easy to verify that
$\|\V_t\|_{r,p} \le \rho\,\|\bX\|_{r,p}$ for any $r,p \ge 1$. In
addition, since any Schatten norm is 
sub-multiplicative we also have that $\|\V_t\|_{S(\infty)} \le
\|\Diag(\tau_t)\|_{S(\infty)}\,\|\bX_t\|_{S(\infty)} \le \rho\,\|\bX_t\|_{S(\infty)}$. 
We therefore obtain the following:
\begin{corollary} \label{cor:mt-online}
  Let $\mc{W}_{1,1},\mc{W}_{2,2},\mc{W}_{2,1},\mc{W}_{S(1)}$ be the
  classes defined in \secref{sec:reg} and lets
  $X_{\infty,\infty},X_{2,2},X_{2,\infty},X_{S(\infty)}$ be the
  radius of $\X$ w.r.t. the corresponding norms.  Then, there exist
  online multi-task learning algorithms with regret bounds according
  to Table \ref{tab:bounds}.
\end{corollary}

Let us now discuss few implications of these bounds, and for
simplicity assume that $k<d$.  Recall that each column of $\bX$
represents the value of a single feature for all the tasks. As
discussed in the previous section, if the matrix $\bX$ is dense and if
we assume that $\W$ is sparse, then using the class $\mc{W}_{1,1}$ is
better than using $\mc{W}_{2,2}$. Such a scenario often happens when
we have many irrelevant features and only are few features that can
predict the target reasonably well. Concretely, suppose that $\bX \in
\{0,1\}^{k \times d}$ and that it typically has $s_x$ non-zero
values. Suppose also that there exists a matrix $\W$ that predicts the
targets of the different tasks reasonably well and has $s_w$ non-zero
values. Then, the bound for $\mc{W}_{1,1}$ is order of $s_w \sqrt{
  \ln(dk) / n}$ while the bound for $\mc{W}_{2,2}$ is order of
$\sqrt{s_w\,s_x /n}$. Thus, $\mc{W}_{1,1}$ will be better if $s_w <
s_x / \ln(dk)$.

Now, consider the class $\mc{W}_{2,1}$. Let us further assume the
following. The non-zero elements of $\W$ are grouped into $s_g$
columns and are roughly distributed evenly over those columns; The
non-zeros of $\bX$ are roughly distributed evenly over the columns.
Then, the bound for $\mc{W}_{2,1}$ is $s_g
\sqrt{(s_w/s_g)\,(s_x/d)\,\ln(d)/n} =
\sqrt{s_g\,s_w\,(s_x/d)\,\ln(d)/n}$.  This bound will be better than
the bound of $\mc{W}_{2,2}$ if $s_g \ln(d) < d$ and will be better
than the bound of $\mc{W}_{1,1}$ if $s_g s_x/d < s_w$. We see that
there are scenarios in which the group norm is better than the
non-grouped norms and that the most adequate class depends on
properties of the problem and our prior beliefs on a good predictor
$\W$.

As to the bound for $\mc{W}_{S(1)}$, it is easy to verify that if the
rows of $\W$ sits in a low dimensional subspace then the spectrum of
$\W$ will be sparse. Similarly, the value of $\|\bX\|_{S(\infty)}$
depends on the maximal singular value of $\bX$, which is likely to be
small if we assume that all the ``energy'' of $\bX$ is spread over its
entire spectrum. In such cases, $\mc{W}_{S(1)}$ can be the best
choice. This is an example of a different type of prior knowledge on
the problem.



\subsection{Batch multi-task learning}

In the batch setting we see a dataset $\mc{T} = \left( (\bX_1,\y_1),\ldots,(\bX_n,\y_n)\right)$
consisting of i.i.d. samples drawn from a distribution $\mc{D}$ over $\X \times \Y$. In the
$k$-task setting, $\X \subseteq \reals^{k \times d}$. Analogous to the single task case, we
define the risk and empirical risk of a multitask predictor $\W \in \reals^{k \times d}$ as:
\begin{align*}
\widehat{L}(\W) &:= \frac{1}{n} \sum_{i=1}^n \sum_{j=1}^k
l^j\left(\inner{\w^j,\bX^j_i},y_i^j\right) 
~~~;~~~
L(\W) := \E_{(\bX,\y)\sim\mc{D}}\left[ \sum_{j=1}^k l^j\left(\inner{\w^j,\bX^j},y^j\right) \right] \ .
\end{align*}
Let $\mc{W}$ be some class of matrices, and define the empirical risk minimizer,
$\widehat{\W} := \mathrm{argmin}_{\W \in \mc{W}} \widehat{L}(\widehat{\W})$. To obtain
excess risk bounds for $\widehat{\W}$, we need to consider the $k$-task Rademacher complexity
\begin{equation*}
\mc{R}^k_{\mc{T}}(\mc{W}) := \E\left[ \sup_{\W \in \mc{W}} \frac{1}{n} \sum_{i=1}^n \sum_{j=1}^k
	\epsilon_i^j \inner{\w^j,\bX^j_i} \right] \ .
\end{equation*}
because, assuming each $l^j$ is $\rho$-Lipschitz, we have the bound
$
	\E\left[ L(\widehat{\W}) - \min_{\W \in \mc{W}} L(\W) \right] \le \rho  \E\left[ \mc{R}^k_{\mc{T}}\left( \mc{W} \right) \right]
$.
This bound follows easily from Talagrand's contraction inequality and Thm. 8 in \cite{Maurer06}.
We can use matrix strong convexity to give the following $k$-task Rademacher bound.
\begin{theorem}
\label{thm:multirad}
(Multitask Generalization)
Suppose $F(\W) \le f_{\max}$ for all $\W \in \mc{W}$
for a function $F$ that is $\sparam$-strongly convex w.r.t. some (matrix) norm $\|\cdot\|$. If the norm $\|\cdot\|_\star$
is invariant under sign changes of the rows of its argument matrix then, for any dataset $\mc{T}$, we have,
$	\mc{R}^k_{\mc{T}}(\mc{W}) \le  X\, \sqrt{ \frac{ 2 f_{\max} }{\sparam n} }
$,
where $X$ is an upper bound on $\|\bX_i\|_\star$.
\end{theorem}
\begin{proof}
We can rewrite $\mc{R}^k_{\mc{T}}(\mc{W}) $ as
\begin{align*}
\E\left[ \sup_{\W \in \mc{W}} \frac{1}{n} \sum_{i=1}^n \sum_{j=1}^k
	\epsilon_i^j \inner{\w^j,\bX^j_i} \right] 
&=      \E\left[ \sup_{\W \in \mc{W}} \frac{1}{n} \sum_{j=1}^k
	\inner{\w^j, \sum_{i=1}^n \epsilon^j_i \bX^j_i} \right]
=       \E\left[ \sup_{\W \in \mc{W}} \frac{1}{n} \inner{ \W, \sum_{i=1}^n \tilde{\bX}_i } \right] \ ,
\end{align*}
where $\tilde{\bX}_i \in\reals^{k \times d}$ is defined by $\tilde{\bX}_i^j = \epsilon_i^j \bX_i^j$ and we have switched to a matrix inner product
in the last line. By the assumption on the dual norm
$\|\cdot\|_\star$, $\| \tilde{\bX}_i \|_\star = \| \bX_i \|_\star \le X$. Now using~\corref{cor:important} and proceeding as in the proof of~\thmref{thm:rademacharbound}, we get, for any $\lambda > 0$,
$
	 \mc{R}^k_{\mc{T}}(\mc{W}) \le \left(\frac{f_{\max}}{\lambda n} + \frac{\lambda X^2}{2\sparam} \right) \ .
$
Optimizing over $\lambda$ proves the theorem.
\end{proof}

Note that both group $(r,p)$-norms and Schatten-$p$ norms satisfy the invariance under row flips mentioned in the theorem above. Thus, we get
the following corollary.

\begin{corollary} \label{cor:mt-batch}
  Let $\mc{W}_{1,1},\mc{W}_{2,2},\mc{W}_{2,1},\mc{W}_{S(1)}$ be the
  classes defined in \secref{sec:reg} and lets
  $X_{\infty,\infty},X_{2,2},X_{2,\infty},X_{S(\infty)}$ be the
  radius of $\X$ w.r.t. the corresponding norms.  Then, the (expected) excess multitask risk
  of the empirical multitask risk minimizer $\widehat{\W}$ satisfies
  the same bounds given in Table \ref{tab:bounds}.
\end{corollary}

\section{Multi-class learning} \label{sec:multi-class}

In this section we consider multi-class categorization problems. We
focus on the online learning model. On round $t$, the online algorithm
receives an instance $x_t \in \reals^d$ and is required to predict its label as a
number in $\{1,\ldots,k\}$. Following the construction of
\cite{CrammerSi00}, the prediction is based on a matrix $\W_t \in
\reals^{k \times d}$ and is defined as the index of the maximal element of
the vector $\W_t x_t$. We use the hinge-loss function adapted to
the multi-class setting. That is, 
\[
l_t(\W_t) ~=~ \max_{r} ( \boldsymbol{1}_{[r \neq y_t]} - (\inner{\w_t^{y_t},x_t} -
\inner{\w_t^r,x_t})) = \max_r ( \boldsymbol{1}_{[r \neq y_t]} - (\inner{\W,\bX_t^{r,y_t}})) ~, 
\]
where $\bX_t^{r,y_t}$ is a matrix with $x_t$ on the $y$'th row, $-x_t$
on the $r$'th row, and zeros in all other elements.
It is easy to verify that $l_t(\W_t)$ upper bounds the zero-one loss,
i.e. if the prediction of $\W_t$ is $r$ then $l_t(\W_t) \ge
\boldsymbol{1}_{[r \neq y_t]}$. 

A sub-gradient of $l_t(\W_t)$ is either a matrix of the form
$-\bX_t^{r,y_t}$ or the all zeros matrix. 
Note that each
column of $\bX_t^{r,y_t}$ is very sparse (contains only two
elements). Therefore,
\[
\|\bX_t^{r,y_t} \|_{\infty,\infty} =  \|x_t\|_\infty ~~;~~
\|\bX_t^{r,y_t} \|_{2,2} =  \sqrt{2}\,\|x_t\|_2 ~~;~~
\|\bX_t^{r,y_t} \|_{2,\infty} =  \sqrt{2}\,\|x_t\|_\infty ~~;~~
\|\bX_t^{r,y_t} \|_{S(\infty)} =  \sqrt{2}\,\|x_t\|_2
\]

Based on this fact, we can easily obtain the following. 
\begin{corollary} \label{cor:mc-online} Let
  $\mc{W}_{1,1},\mc{W}_{2,2},\mc{W}_{2,1},\mc{W}_{S(1)}$ be the
  classes defined in \secref{sec:reg} and let $X_2 = \max_t \|x_t\|_2$
  and $X_\infty = \max_t \|x_t\|_\infty$.  Then, there exist online
  multi-class learning algorithms with regret bounds given by the
  following table
\begin{center}
\begin{tabular}{l|c|c|c|c}
class & $\mc{W}_{1,1}$ & $\mc{W}_{2,2}$ &$\mc{W}_{2,1}$ &
$\mc{W}_{S(1)}$ \\[0.1cm] \hline
& & & & \\[0.01cm]
bound  & $W_{1,1}\,X_{\infty} \sqrt{\tfrac{\ln(kd)}{ n}}$ &
$W_{2,2}\,X_{2}\,\sqrt{\tfrac{1}{n} }$ &
$W_{2,1}\,X_{\infty}\,\sqrt{\tfrac{\ln(d)}{ n}}$ &
$W_{S(1)}\,X_{2}\,\sqrt{\tfrac{\ln(\min\{d,k\})}{ n}}$ \\[0.1cm]
\end{tabular}
\end{center}
\end{corollary}
Let us now discuss the implications of this bound. First, if $X_2
\approx X_\infty$, which will happen if instance vectors are sparse,
then $\mc{W}_{1,1}$ and $\mc{W}_{2,1}$ will be inferior to
$\mc{W}_{2,2}$. In such a case, using $\mc{W}_{S(1)}$ can be even
better if $\W$ sits in a low dimensional space but each row of $\W$
still has a unit norm. Using $\mc{W}_{S(1)}$ in such a case was
previously suggested by \cite{AmitFiSrUl07}, who observed that
empirically, the class $\mc{W}_{S(1)}$ performs better than
$\mc{W}_{2,2}$ when there is a shared structure between classes. The
analysis given in \corref{cor:mc-online} provides a first rigorous
explanation to such a behavior. 

Second, if $X_2$ is much larger than $X_\infty$, and if columns of
$\W$ share common sparsity pattern, then $\mc{W}_{2,1}$ can be factor
of $\sqrt{k}$ better than $\mc{W}_{1,1}$ and factor of $\sqrt{d}$
better than $\mc{W}_{2,2}$. To demonstrate this, let us assume that each
vector $x_t$ is in $\{\pm 1\}^d$ and it represents experts advice of
$d$ experts. Therefore, $X_2 = \sqrt{d}\,X_\infty$. Next, assume that
a combination of the advice of $s \ll d$ experts predicts very well the
correct label (e.g., the label is represented by the binary number
obtained from the advice of $s=\log(k)$ experts). In that case, $W$
will be a matrix such that all of its columns will be $0$ except
$s$ columns which will take values in $\{\pm 1\}$. The bounds for
$\mc{W}_{1,1},\mc{W}_{2,2},$ and $\mc{W}_{2,1}$ in that case becomes 
$ks\sqrt{\ln(kd)},~\sqrt{ksd},$ and $\sqrt{ks\ln(d)}$ respectively. 
That is, $\mc{W}_{2,1}$ is a factor
of $\sqrt{ks}$ better than $\mc{W}_{1,1}$ and a factor of $\sqrt{d}$
better than $\mc{W}_{2,2}$ (ignoring logarithmic terms). The class
$\mc{W}_{S(1)}$ will also have a dependent on $\sqrt{d}$ in such a
case and thus it will be much worse than $\mc{W}_{2,2}$ when $d$ is
large. 

For concreteness, we now utilize our result for deriving a
group Multi-class Perceptron algorithm. To the best of our knowledge, this
algorithm is new, and based on the discussion above, it should
outperform both the multi-class Perceptron of \cite{CrammerSi00} as
well as the vanilla application of the $p$-norm Perceptron framework of
\cite{Gentile03,GroveLiSc01} for multi-class categorization. 

The algorithm is a specification of the general online mirror descent
procedure (Algorithm \ref{alg:frl}) with $f(\W) =
\thalf\|\W\|_{2,r}^2$, $r = \log(d)/(\log(d)-1)$, and with a conservative update (i.e., we ignore
rounds on which no prediction mistake has been made). Recall that the
Fenchel dual function is $f^\star(\V) = \thalf \|\V\|_{2,p}^2$ where
$p = (1-1/r)^{-1} = \log(d)$. The $(i,j)$ element of the gradient of $f^\star$
is
\begin{equation} \label{eqn:nablaF}
(\nabla f^\star(\V))_{i,j} ~=~
\frac{\|\V^j\|_2^{p-2}}{\|\V\|_{2,p}^{p-2}} \, V_{i,j} ~.
\end{equation}

\begin{algorithm}
\caption{Group Multi-class Perceptron}
\label{alg:grp_percep}
\begin{algorithmic}
\STATE $p=\log d$
\STATE $\V_1 = \zero \in \reals^{k \times d}$
\FOR{$t = 1, \ldots, T$}
       \STATE Set $\W_t = \nabla f^\star(\V_t)$ (as defined in \eqref{eqn:nablaF})
	\STATE Receive $\x_t \in \reals^d$
	\STATE $\yhat_t = \arg\max_{r\in[k]}\ \left( \W_t \x_t \right)_r$
	\STATE Predict $\yhat_t$ and receive true label $y_t$
	\STATE $\U_t \in \reals^{k\times d}$ is the matrix with $\x_t$ in the $\yhat_t$ row and $-\x_t$ in the $y_t$ row
	\STATE Update: $\V_{t+1} = \V_t -  \U_t$
\ENDFOR
\end{algorithmic}
\end{algorithm}

To analyze the performance of Algorithm \ref{alg:grp_percep}, let $I
\subseteq [n]$ be the set of rounds on which the algorithm made a
prediction mistake. Note that the above algorithm is equivalent (in
terms of the number of mistakes) to an algorithm that performs the
update $\V_{t+1} = \V_t + \eta \U_t$ for any $\eta$ (see
\cite{Gentile03}). Therefore, we can apply our general online regret
bound (\corref{cor:online21}) on the sequence of examples in $I$ we
obtain that for any $\W$
\[
\sum_{t \in I} l_t(\W_t) - \sum_{t \in I} l_t(\W) ~\le~
O\left( X_\infty\,\|\W\|_{2,1}\,\sqrt{\log(d) \, |I|} \right)~.
\]
Recall that $l_t(\W_t)$ upper bounds the zero-one error and therefore
the above implies that 
\[
|I| - \sum_{t \in I} l_t(\W) ~\le~
O\left( X_\infty\,\|\W\|_{2,1}\,\sqrt{\log(d) \, |I|} \right)~.
\]
Solving for $|I|$ we conclude that:
\begin{corollary}
The number of mistakes Algorithm \ref{alg:grp_percep} will
make on any sequence of examples for which $\|x_t\|_\infty \le
X_\infty$ is upper bounded by
\[
\min_{\W} ~~
\sum_{t} l_t(\W) + O\left(
X_\infty\,\|\W\|_{2,1}\,\sqrt{\log(d) \, \sum_t L_t(\W) } 
\right) ~~.
\]
\end{corollary}

\section{Kernel learning} \label{sec:kernels}

We briefly review the kernel learning setting first explored in \cite{LanckrietCrBaElJo04}.
Let $\mc{X}$ be an input space and let
$\dataset = (\x_1, \ldots, \x_n) \in \mc{X}^n$ be the training dataset.
Kernel algorithms work with the space of linear functions,
$
\left\{ \x \mapsto \sum_{i=1}^n \alpha_i K(\x_i,\x) \::\: \alpha_i \in \reals \right\}
$.
In kernel learning, we consider a kernel {\em family} $\mc{K}$ and consider the class,
$
\left\{ \x \mapsto \sum_{i=1}^n \alpha_i K(\x_i,\x) \::\: K \in \mc{K},\ 
\alpha_i \in \reals \right\}
$.
In particular, we can choose a finite set $\{K_1,\ldots,K_k\}$ of base kernels and
consider the convex combinations,
$
\mc{K}_c^+ = \left\{ \sum_{j=1}^k \mu_j K_j \::\: \mu_j \ge 0,\ \sum_{j=1}^k \mu_j = 1 \right\}\ .
$
This is the unconstrained function class. In applications, one constrains the function
class in some way. The class considered in \cite{LanckrietCrBaElJo04} is
\begin{equation}
\label{eqn:convexcomb}
\mc{F}_{\convexK} = \left\{ \x \mapsto \sum_{i=1}^n \alpha_i K(\x_i,\cdot) \::\:
K = \sum_{j=1}^k \mu_j K_j,\ \mu_j \ge 0, \right. 
\left. \sum_{j=1}^k \mu_j = 1,\ \valpha^\top K(\dataset) \valpha \le 1/\gamma^2 \right\}
\end{equation}
where $\gamma > 0$ is a margin parameter and $K(\dataset)_{i,j} = K(\x_i,\x_j)$ is the 
Gram matrix of $K$ on the dataset $\dataset$.

\begin{theorem} (Kernel learning) \label{thm:kernel}
Consider the class $\mc{F}_{\convexK}$ defined in \eqref{eqn:convexcomb}. Let $K_j(\x,\x) \le
B$ for $1\le j \le k$ and $\x \in \mc{X}$. Then,
$
	\mc{R}_\dataset(\mc{F}_{\convexK}) \le e\sqrt{\frac{ B \log k }{\gamma^2 n}}\ .
$
\end{theorem}
The proof follows directly from the equivalence between kernel
learning and group Lasso \cite{Bach08}, and then applying our bound on the
class $\mc{W}_{2,1}$.  For completeness, we give a rigorous proof in
the appendix.

Note that the dependence on the
number of base kernels, $k$, is rather mild (only logarithmic) ---
implying that we can learn a kernel as a (convex) combination of a rather large number of base kernels.
Also, let us discuss how the above improves upon the prior bounds
provided by \cite{LanckrietCrBaElJo04} and \cite{SrebroBe06} (neither
of which had logarithmic $k$ dependence). The former proves a bound of
$O\left( \sqrt{\frac{B k}{\gamma^2 n}} \right)$ which is quite inferior to our
bound. We cannot compare our bound directly to the bound in \cite{SrebroBe06} as
they do not work with Rademacher complexities. However, if one compares the
resulting generalization error bounds, then their bound is
$
O\left(\sqrt{
\frac{ k\log \frac{n^3 B}{\gamma^2 k} + \frac{B}{\gamma^2} \log \frac{\gamma
n}{\sqrt{B}}
\log \frac{n B}{\gamma^2} }
{n}
}\right)
$
and ours is
$
O\left( \sqrt{ \frac{B \log k}{\gamma^2 n} }\right)
$.
If $k \ge n$, their bound is vacuous (while ours is still meaningful).
If $k \le n$, our bound is better.

Finally, we note that recently \cite{YingCa09} devoted a dedicated effort to
derive a result similar to \thmref{thm:kernel} using a Rademacher chaos
process of order two over candidate kernels. In contrast to their
proof, our result seamlessly follows from the general framework of
deriving bounds using the strong-convexity/strong-smoothness duality.


\section*{Acknowledgements}
We thank Andreas Argyriou, Shmuel Friedland \& Karthik Sridharan for helpful discussions.

{\small
\bibliography{bib,addbib}
}

\newpage

\appendix

\section{Convex Analysis and Matrix Computation} \label{sec:convex}

\subsection{Convex analysis}
We briefly recall some key definitions from convex analysis that
are useful throughout the paper (for details, see any of the several excellent references on the subject, e.g. \cite{BorweinLewis06,Rockafellar70}).
We consider convex functions $f: \X \to \reals \cup \{\infty\}$,
where $\X$ is a Euclidean vector space equipped with an inner product
$\inner{\cdot, \cdot}$. We denote $\Ereals = \reals \cup \{\infty\}$.
Recall that the subdifferential of $f$
at $x \in \X$, denoted by $\partial f(x)$, is defined as
$
	\partial f(x) := \{ y \in \X \;:\; 
	\forall z,\ f(x+z) \ge f(x) + \inner{y, z} \}
$.
The Fenchel conjugate $f^\star:\X \to \Ereals$ is defined as
$
	f^\star(y) := \sup_{x \in \X} \inner{x, y} - f(x)
$.

We also deal with a variety of norms in this paper. Recall that given a norm $\|\cdot\|$ on $\X$, its
dual norm is defined as
$
	\| y \|_\star := \sup\{ \inner{ x, y } \;:\; \|x\| \le 1 \}
$.
An important property of the dual norm is that the Fenchel conjugate of the
function $\frac{1}{2}\|x\|^2$ is $\frac{1}{2}\|y\|_\star^2$.

The definition of Fenchel conjugate implies that for any $x,y$,
$
	f(x) + f^\star(y) \ge \inner{x, y}
$,
which is known as the Fenchel-Young inequality. An equivalent and useful
definition of the subdifferential can be given in terms of the Fenchel
conjugate:
$
	\partial f(x) = \{ y \in \X \;:\; 
	f(x) + f^*(y) = \inner{x, y} \}
$.

\subsection{Convex analysis of matrix functions} 
\label{subsct:matrix_calc}

We consider the vector space $\X=\reals^{m \times n}$ of real matrices of size
$m \times n$ and the
vector space $\X=\sym^n$ of symmetric matrices of size $n \times n$, both equipped with the inner product,
$
	\inner{\bX,\bY} := \Trace(\bX^\top \bY) 
$.
Recall that any matrix $\bX \in \reals^{m \times n}$ can be decomposed as
$
	\bX = \U \Diag(\sigma(\bX)) \V
$
where $\sigma(\bX)$ denotes the vector $(\sigma_1,\sigma_2,
\ldots \sigma_l)$ ($l=\min\{m,n\}$), where $\sigma_1 \ge \sigma_2 \ge \ldots \ge \sigma_l \ge 0$ 
are the singular values of $\bX$ arranged in non-increasing order, and
$\U\in\reals^{m\times m},\V\in\reals^{n\times n}$ are orthogonal matrices.
Also, any matrix $\bX \in \sym^n$ can be decomposed as,
$
	X = \U \Diag(\lambda(\bX)) \U^\top
$
where $\lambda(\bX)=(\lambda_1, \lambda_2, \ldots \lambda_n)$, where $\lambda_1 \ge \lambda_2 \ge \ldots \ge \lambda_n$ are the eigenvalues of
$\bX$ arranged in non-increasing order, and $\U$ is an orthogonal matrix. Two
 important results relate matrix inner products to inner products
between singular (and eigen-) values

\begin{theorem}
{\bf (von Neumann)}
Any two matrices $X,Y \in \reals^{m \times n}$ satisfy the inequality
$$
	\inner{\bX,\bY} \le \inner{\sigma(\bX),\sigma(\bY)}\ .
$$
Equality holds above, if and only if, there exist orthogonal $\U, \V$ such that
\begin{align*}
\bX &= \U\Diag(\sigma(\bX))\V &
\bY &= \U\Diag(\sigma(\bY))\V\ .
\end{align*}
\end{theorem}

\begin{theorem}
{\bf (Fan)}
Any two matrices $\bX,\bY \in \sym^n$ satisfy the inequality
$$
	\inner{\bX,\bY} \le \inner{\lambda(\bX),\lambda(\bY)}\ .
$$
Equality holds above, if and only if, there exists orthogonal $\U$ such that
\begin{align*}
\bX &= \U\Diag(\lambda(\bX))\U^\top &
\bY &= \U\Diag(\lambda(\bY))\U^\top\ .
\end{align*}
\end{theorem}

We say that a function $g:\reals^n \to \Ereals$ is symmetric if $g(x)$ is invariant
under arbitrary permutations of the components of $x$. We say $g$ is absolutely
symmetric if $g(x)$ is invariant under arbitrary permutations and sign changes
of the components of $x$.

Given a function $f:\reals^l \to \Ereals$, we can define a function $f \circ
\sigma:\reals^{m \times n} \to \Ereals$ as,
$$
	(f \circ \sigma)(\bX) := f(\sigma(\bX))\ .
$$
Similarly, given a function $g:\reals^n \to \Ereals$, we can define a function $g
\circ \lambda:\sym^n \to \Ereals$ as,
$$
	(g \circ \lambda)(\bX) := g(\lambda(\bX))\ .
$$ 
This allows us to define functions over matrices starting from functions over
vectors. Note that when we use $f \circ \sigma$ we are assuming that $\X=\reals^{m \times n}$ and for $g \circ \lambda$ we have $\X=\sym^n$. The following result allows us to immediately compute the conjugate of $f \circ \sigma$ and $g \circ \lambda$ in terms of the conjugates of $f$ and $g$
respectively.

\begin{theorem}
\label{thm:conjugate}
(\cite{Lewis95b}) Let $f:\reals^l \to \Ereals$ be an absolutely symmetric function. Then,
$$
	(f \circ \sigma)^\star = f^\star \circ \sigma\ .
$$
Let $g:\reals^n \to \Ereals$ be a symmetric function. Then,
$$
	(g \circ \lambda)^\star = g^\star \circ \lambda\ .
$$
\end{theorem}
\begin{proof}
\cite{Lewis95b} proves this for singular values. For the
eigenvalue case, the proof is entirely analogous to that in \cite{Lewis95b}, except
that Fan's inequality is used instead of von Neumann's inequality.
\end{proof}

Using this general result, we are able to define certain matrix norms.

\begin{corollary}
\label{cor:conjugate} (Matrix norms)
Let $f:\reals^l \to \Ereals$ be absolutely symmetric. Then
if $f = \|\cdot\|$ is a norm on $\reals^l$ then $f \circ \sigma = \|\sigma(\cdot)\|$ is a norm on $\reals^{m \times n}$.
Further, the dual of this norm is $\|\sigma(\cdot)\|_\star$.

Let $g:\reals^n \to \Ereals$ be symmetric. Then if $g = \|\cdot\|$ is a norm on $\reals^n$ then $g \circ \lambda = \|\lambda(\cdot)\|$ is a norm on $\sym^n$. Further, the dual of this norm is $\|\lambda(\cdot)\|_\star$.
\end{corollary}

Another nice result allows us to compute subdifferentials of $f \circ \sigma$ and
$g \circ \lambda$ (note that elements in the subdifferential of $f \circ \sigma$
and $g \circ \lambda$ are matrices) from the subdifferentials of $f$ and $g$
respectively.

\begin{theorem}
\label{thm:gradient}
(\cite{Lewis95b}) Let $f:\reals^l \to \Ereals$ be absolutely symmetric and $\bX \in \reals^{m \times
n}$. Then,
\begin{equation*}
	\partial (f \circ \sigma) (\bX)
	= \{ \U\Diag(\mu)\V^\top \;:\;
		\mu \in \partial f(\sigma(\bX)) 
		\U,\V \text{ orthogonal},\ 
		\bX = \U\Diag(\sigma(\bX))\V^\top
	\}
\end{equation*}
Let $g:\reals^n \to \Ereals$ be symmetric and $\bX \in \sym^n$. Then,
\begin{equation*}
	\partial (g \circ \lambda) (\bX)
	= \{ \U\Diag(\mu)\U^\top \;:\;
		\mu \in \partial g(\lambda(\bX)) 
		\U \text{ orthogonal},\ 
		\bX = \U\Diag(\lambda(X))\U^\top
	\}
\end{equation*}
\end{theorem}
\begin{proof}
Again, \cite{Lewis95b} proves the case for singular values. For the
eigenvalue case, again, the proof is identical to that in \cite{Lewis95b}, except
that Fan's inequality is used instead of von Neumann's inequality.
\end{proof}

%

\section{Technical Proofs}

\subsection{Proof of \thmref{thm:strongsmooth}}

First, \cite[Lemma 15]{Shalev07} yields one half of the claim ($f$ strongly convex $\Rightarrow$ $f^\star$ strongly smooth). 
It is left to prove that $f$ is strongly convex assuming that
$f^\star$ is strongly smooth. For simplicity assume that
$\sparam = 1$.  Denote $g(y) = f^\star(x+y) - 
(f^\star(x) + \inner{\nabla f^\star(x),y})$. By the smoothness assumption, $g(y) \le \half \|y\|_\star^2$. This implies
that $g^\star(a) \ge \half \|a\|^2$ because of  
\cite[Lemma 19]{ShalevSi08} and that the conjugate of half squared norm
is half squared of the dual norm. Using the definition of $g$ we have
\begin{align*}
g^\star(a) &= \sup_y \inner{y,a} - g(y) \\
&= \sup_y \inner{y,a} - \left( f^\star(x+y) - \left(f^\star(x) + \inner{\nabla f^\star(x),y}\right) \right) \\
&= \sup_y \inner{y,a+\nabla f^\star(x)} - f^\star(x+y) + f^\star(x) \\
&= \sup_z \inner{z-x,a+\nabla f^\star(x)} - f^\star(z) + f^\star(x) \\
&= f(a + \nabla f^\star(x)) + f^\star(x) - \inner{x,a+\nabla f^\star(x)}
\end{align*} 
where we have used that $f^{\star \star}=f$, in the last step.
Denote $u = \nabla f^\star(x)$. From the equality in Fenchel-Young
(e.g.  \cite[Lemma 17]{ShalevSi08}) we obtain that $\inner{x,u} =
f^\star(x) + f(u)$ and thus
$$
g^\star(a) = f(a+u)-f(u) - \inner{x,a} ~.
$$
Combining with $g^\star(a) \ge \half \|a\|^2$, we have
\begin{equation} \label{eqn:fau}
f(a+u) - f(u) - \inner{x,a} ~\ge~ \half \|a\|^2 ~~,
\end{equation}
which holds for all $a,x$, with $u = \nabla f^\star(x)$.  

Now let us prove that for any point $u'$ in the relative interior 
of the domain of $f$ that if $x \in \partial f(u')$ then $u'=\nabla
f^\star (x)$. Let $u := \nabla f^\star(x)$ and we must show that $u'=u$. 
By Fenchel-Young, we have that $\inner{x,u'} =
f^\star(x) + f(u')$, and, again by Fenchel-Young (and $f^{\star
  \star}=f$), we have $\inner{x,u} = f^\star(x) + f(u)$. We can now
apply Equation~\eqref{eqn:fau}, to obtain:
\begin{align*}
0 &= \inner{x,u} - f(u) - \left(\inner{x,u'} -
  f(u')\right) \\
&= f(u')-f(u)-\inner{x,u'-u} 
~\ge~ \half \|u'-u\|^2 ~,
\end{align*}
which implies that $u'=\nabla
f^\star (x)$.

Next, let $u_1,u_2$ be two points in the relative interior 
of the domain of $f$, let 
$\alpha \in (0,1)$, and let $u = \alpha u_1 + (1-\alpha) u_2$.
Let $x \in \partial f(u)$ (which is non-empty~\footnote{The set
  $\partial f(u)$ is not empty for all $u$ in the 
relative interior of the domain of $f$. 
See the relative max formula in \cite[page 42]{BorweinLewis06}
or \cite[page 253]{Rockafellar70}. 
If $u$ is not in the interior of $f$, then $\partial f(u)$ 
is empty. But, a function is defined to be essentially strictly convex
if it is strictly convex on any subset of $\{u : \partial f(u) \neq \emptyset\}$.
The last set is called the  domain of $\partial f$ and it contains
the relative interior of the domain of $f$, so we're ok here.}).
We have that $u = \nabla f^\star(x)$, by the previous argument.
Now we are able to apply Equation~\eqref{eqn:fau} twice, once 
with $a =u_1 -u$ and once with $a = u_2 - u$ (and both with $x$) to obtain
\begin{eqnarray*}
f(u_1) - f(u) - \inner{x,u_1-u} &\ge& \half \|u_1-u\|^2 \\
f(u_2) - f(u) - \inner{x,u_2-u} &\ge& \half \|u_2-u\|^2 
\end{eqnarray*}
Finally, summing up the above two equations with coefficients 
$\alpha$ and $1-\alpha$ we obtain that $f$ is strongly convex.

\subsection{Proof of \thmref{thm:grpsmoothness}}

Note that an equivalent definition of $\sigma$-smoothness of a function $f$
w.r.t. a norm $\|\cdot\|$ is that, for all $x,y$ and $\alpha \in [0,1]$, we have
\begin{equation*}
	f(\alpha x + (1-\alpha) y) \ge
	\alpha f(x) + (1-\alpha) f(y)
	- \half \sigma \alpha (1-\alpha) \|x-y\|^2\ .
\end{equation*}
Let $\bX,\bY\in \reals^{m \times n}$ be arbitrary matrices with columns
$\bX^i$ and $\bY^i$ respectively. We need to prove
\begin{equation}
\label{eqn:grpsmoothness}
\|(1-\alpha) \bX + \alpha \bY\|^2_{\colnorm,\rownorm}
\ge \alpha \| \bX \|^2_{\colnorm,\rownorm}
+ (1-\alpha) \| \bY \|^2_{\colnorm,\rownorm}
- \half (\sigma_1 + \sigma_2) \alpha (1-\alpha)
\| \bX - \bY \|^2_{\colnorm,\rownorm} \ .
\end{equation}
Using smoothness of $\colnorm$ and that $\rownorm$ is a Q-norm, we have,
\begin{align}
\nonumber
\|(1-\alpha) \bX + \alpha \bY\|^2_{\colnorm,\rownorm} &= (\rownorm^2 \circ \sqrt{})(\ldots,  \colnorm^2(\alpha \bX^i + (1-\alpha) \bY^i) , \ldots) \\
\nonumber
&\ge (\rownorm^2 \circ \sqrt{})(\ldots, \alpha \colnorm^2(\bX^i) + (1-\alpha)
\colnorm^2(\bY^i) - \half \sigma_1 \alpha(1-\alpha) \colnorm^2(\bX^i-\bY^i) , \ldots)
\\
\nonumber
&\ge (\rownorm^2 \circ \sqrt{})(\ldots, \alpha \colnorm^2(\bX^i) + (1-\alpha)
\colnorm^2(\bY^i), \ldots) \\
\nonumber
&\quad -\half \sigma_1 \alpha(1-\alpha) (\rownorm^2 \circ \sqrt{})(\ldots,\colnorm^2(\bX^i
- \bY^i),\ldots) \\
\label{eqn:grpsmooth.intermediate}
&=\rownorm^2(\ldots, \sqrt{ \alpha \colnorm^2(\bX^i) + (1-\alpha)
\colnorm^2(\bY^i) }, \ldots ) -\half \sigma_1 \alpha(1-\alpha) \|\bX-\bY\|^2_{\colnorm,\rownorm} \ .
\end{align}
Now, we use that, for any $x,y \ge 0$ and $\alpha \in [0,1]$, we have
$\sqrt{ \alpha x^2 + (1-\alpha) y^2 } \ge \alpha x + (1-\alpha) y$.
Thus, we have
\begin{align*}
&\quad \rownorm^2(\ldots, \sqrt{ \alpha \colnorm^2(\bX^i) + (1-\alpha)
\colnorm^2(\bY^i) }, \ldots ) \\
&\ge \rownorm^2(\ldots, \alpha \colnorm(\bX^i) + (1-\alpha) \colnorm(\bY^i), \ldots)
\\
&\ge \alpha \rownorm^2(\ldots, \colnorm(\bX^i), \ldots)
+ (1-\alpha) \rownorm^2(\ldots, \colnorm(\bY^i), \ldots) \\
&\quad - \half \sigma_2 \alpha (1-\alpha) \rownorm^2(\ldots, \colnorm(\bX^i) -
\colnorm(\bY^i), \ldots) \\
&\ge \alpha \|\bX\|^2_{\colnorm,\rownorm} + (1-\alpha) \|\bY\|^2_{\colnorm,\rownorm}
- \half \sigma_2 \alpha (1-\alpha) \rownorm^2(\ldots, \colnorm(\bX^i-\bY^i),
\ldots) \\
&= \alpha \|\bX\|^2_{\colnorm,\rownorm} + (1-\alpha) \|\bY\|^2_{\colnorm,\rownorm}
- \half \sigma_2 \alpha (1-\alpha) \|\bX-\bY\|^2_{\colnorm,\rownorm}
\end{align*}
Plugging this into \eqref{eqn:grpsmooth.intermediate} proves
\eqref{eqn:grpsmoothness}.

\subsection{Proof of \thmref{thm:kernel}}
Let $\mc{H}_j$ be the RKHS of $K_j$,
$
	\mc{H}_j = \left\{ \sum_{i=1}^l \alpha_i K_j(\tx_i,\cdot) \::\: l > 0,\
	\tx_i \in \mc{X},\ \valpha \in \reals^l \right\}
$
equipped with the inner product
\begin{equation*}
	\inner{ \sum_{i=1}^l \alpha_i K_j(\tx_i,\cdot), \sum_{j=1}^m \alpha'_i K_j(\tx'_j,\cdot)}_{\mc{H}_j} =
	\sum_{i,j} \alpha_i \alpha'_{j} K_j(\tx_i,\tx'_{j})
\end{equation*}
Consider the space $\mc{H} = \mc{H}_1\times \ldots \times\mc{H}_k$ equipped with
the inner product
$
	\inner{\vec{u},\vec{v}} := \sum_{i=1}^k \inner{u_i,v_i}_{\mc{H}_i}
$.
For $\vec{w} \in \mc{H}$, let $\|\cdot\|_{2,1}$ be the norm defined by
$
	\|\vec{w}\|_{2,1} =  \sum_{i=1}^k \|w_i\|_{\mc{H}_i}\ .
$
It is easy to verify that 
$\mc{F}_{\convexK} \subseteq \mc{F}_r$
where
$	\mc{F}_r := \{ \x \mapsto \inner{\vec{w}, \vec{\phi}(\x)} \::\:
	\vec{w} \in \mc{H},\ \|\vec{w}\|_{2,1} \le 1/\gamma \}\ ,
$
and
$
	\vec{\phi}(\x) = (K_1(\x,\cdot),\ldots,K_k(\x,\cdot)) \in \mc{H}
$.
Since $\|K_j(\x,\cdot)\|_{\mc{H}_j} \le \sqrt{B}$, we also have $\|\vec{\phi}(\x)\|_{2,s}
\le k^{1/s}\sqrt{B}$ for any $\x \in \mc{X}$. The claim now follows
directly from the results we derived in \secref{sec:preliminaries}.

\end{document}

%% file: intro.tex
As we tackle more challenging learning problems, there is an
increasing need for algorithms that efficiently impose more
sophisticated forms of prior knowledge. Examples include: the group
Lasso problem (for ``shared" feature selection across problems),
kernel learning, multi-class prediction, and multi-task learning. A
central question here is to understand the performance of such
algorithms in terms of the attendant complexity restrictions imposed
by the algorithm. Such analyses often illuminate the nature in which
our prior knowledge is being imposed.

The predominant modern method for imposing complexity restrictions is
through regularizing a vector of parameters, and much work has gone
into understanding the relationship between the nature of the
regularization and the implicit prior knowledge imposed, particular
for the case of regularization with $\ell_2$ and $\ell_1$ norms (where
one is more tailored to rotational invariance and margins, while the
other is more suited to sparsity).  When dealing with more complex
problems, we need systematic tools for designing more complicated
regularization schemes. This work examines regularization based on
group norms and spectral norms of \emph{matrices}. We analyze the
performance of such regularization methods and provide a methodology
for choosing a regularization function based on the underlying
statistical properties of a given problem.

In particular, we utilize a recently developed methodology, based on
the notion of \emph{strong} convexity, for designing and analyzing the
regret or generalization ability of a wide range of learning
algorithms (see e.g. \cite{Shalev07,KakadeSrTe08}).  In fact, most of
our efficient algorithms (both in the batch and online settings)
impose some complexity control via the use of some \emph{strictly}
convex penalty function either explicitly via a regularizer or
implicitly in the design of an online update rule. Central to
understanding these algorithms is the manner in which these penalty
functions are strictly convex, i.e. the behavior of the ``gap'' by
which these convex functions lie above their tangent planes, which is
strictly positive for strictly convex functions. Here, the notion of
strong convexity provides one means to characterize this gap in terms
of some general norm rather than just Euclidean.

The importance of strong convexity can be understood using the duality
between strong convexity and strong smoothness. Strong smoothness
measures how well a function is approximated at some point by its
linearization. Linear functions are easy to manipulate (e.g. because
of the linearity of expectation). Hence, if a function is sufficiently
smooth we can more easily control its behavior. We further distill the
analysis given in \cite{Shalev07,KakadeSrTe08} --- based on the
strong-convexity/smoothness duality, we derive a key inequality which
seamlessly enables us to design and analyze a family of learning
algorithms.

Our focus in this work is on learning with matrices. We characterize
a number of matrix based regularization functions, of recent interest,
as being strongly convex functions --- allowing us to immediately
derive learning algorithms by relying on the family of learning
algorithms mentioned previously. Specifying the general performance
bounds for the specific matrix based regularization method, we are
able to systematically decide which regularization function is more
appropriate based on underlying statistical properties of a given problem.

\subsection{Our Contributions}

We can summarize the contributions of this work as follows:

\begin{itemize}
\item We show how the framework based on strong convexity/strong
  smoothness duality (see \cite{Shalev07,KakadeSrTe08}) provides a
  methodology for analyzing matrix based learning methods, which are
  of much recent interest.  These results reinforce the usefulness of
  this framework in providing both learning algorithms, and their
  associated complexity analysis. For this reason, we further distill
  the analysis given in \cite{Shalev07,KakadeSrTe08} by emphasizing a
  key inequality which immediately enables us to design and analyze a
  family of learning algorithms.

\item We provide template algorithms (both in the online and batch
  settings) for a number of machine learning problems of recent
  interest, which use matrix parameters. In particular, we provide a
  simple derivation of generalization/mistake bounds for:
  (i) online and batch multi-task learning using group or spectral
  norms, (ii) online multi-class categorization using group or spectral
  norms, and (iii) kernel learning.

\item Based on the derived bounds, we interpret how statistical
  properties of a given problem can help us decide which
  regularization function is appropriate. For example, for the case of
  multi-class learning, we describe and analyze a new ``group
  Perceptron'' algorithm and show that with a shared structure between
  classes, this algorithm significantly outperforms previously
  proposed algorithms. Similarly, for the case of multi-task learning,
  the pressing question is what shared structure between the tasks
  allows for sample complexity improvements and by how much? We
  discuss these issues based on our regret and generalization bounds.

\item Our unified analysis significantly simplifies previous analyses of
  recently proposed algorithms. For example, the generality of this
  framework allows us to simplify the proofs of previously
  proposed regret bounds for online multi-task learning
  (e.g. \cite{CavallantiBaGe08,AgarwalRaBa08}). Furthermore, bounds that
  follow immediately from our analysis are sometimes much sharper than
  previous results (e.g. we improve the bounds for multiple
  kernel learning given in \cite{LanckrietCrBaElJo04,SrebroBe06}).

\end{itemize}

\subsection{Related work}
We first discuss related work on learning with matrix parameters then
discuss the use of strong convexity in learning.

\vspace{0.1in}
\noindent {\bf Matrix Learning: }
This is growing body of work studying learning problems in which the
parameters can be organized as matrices. Several examples are
multi-class categorization (e.g. \cite{CrammerSi00}), multi-task and
multi-view learning (e.g. \cite{CavallantiBaGe08,AgarwalRaBa08}), and
online PCA \citep{WarmuthKu06}. It was also studied under the
framework of group Lasso
(e.g. \cite{YuanLi06,ObozinskiTaJo07,Bach08}).

In the context of learning vectors (rather than matrices), the 
study of the relative performance of different regularization
techniques based on properties of a given task dates back to
\cite{Littlestone88,KivinenWa97}. In the context of batch learning, it
was studied by several authors (e.g. \cite{Ng04}). 


We also note that much of the work on multi-task learning for
regression is on union support recovery --- a setting where the
generative model specifies a certain set of relevant features (over
all the tasks), and the analysis here focuses on the conditions and
sample sizes under which the union of the relevant features can be
correctly identified (e.g.~\cite{ObozinskiTaJo07,Lounici09}). Essentially, this is
a generalization of the issue of identifying the relevant feature set
in the standard single task regression setting, under $\ell_1$
regression. In contrast, our work focuses on the agnostic
setting of just understanding the sample size needed to obtain a given
error rate (rather than identifying the relevant features themselves).

We also discuss related work on kernel learning in
\secref{sec:kernels}. Our analysis here utilizes the equivalence
between kernel learning and group Lasso (as noted in \cite{Bach08}).

\vspace{0.1in}
\noindent {\bf Strong Convexity/Strong Smoothness: }
The notion of {\em strong convexity} takes its roots in
optimization. \cite{Zalinescu02} attributes it to a paper of Polyak in
the 1960s. Relatively recently, its use in machine learning has been
two fold: in deriving regret bounds for online algorithms and
generalization bounds in batch settings.

The duality of strong convexity and strong smoothness was first used
by \cite{ShalevSi06c,Shalev07} in the context of deriving low regret online
algorithms.  Here, once we choose a particular strongly convex penalty
function, we immediately have a family of algorithms along with a
regret bound for these algorithms that is in terms of a certain strong
convexity parameter. A variety of algorithms (and regret bounds) can
be seen as special cases.

A similar technique, in which the Hessian is directly
bounded, is described by \cite{GroveLiSc01,ShalevSi07MLJ}.
Another related approach involved bounding a Bregman divergence
\citep{KivinenWa97,KivinenWa01,Gentile03} (see 
\cite{CesaBianchiLu06} for a detailed survey). 
Another interesting application of the very same duality is for
deriving and analyzing boosting algorithms \citep{ShalevSi08}.

More recently, \cite{KakadeSrTe08} showed how to use the very same duality
for bounding the Rademacher complexity of classes of linear
predictors. That the Rademacher complexity is closely related to Fenchel
duality was shown in \cite{MeirZh03}, and the work in
\cite{KakadeSrTe08} made the further connection to strong
convexity. Again, under this characterization, a number of generalization
and margin bounds (for methods which use linear prediction) are
immediate corollaries, as one only needs to specify the strong convexity
parameter from which these bounds easily follow (see \cite{KakadeSrTe08} for details).

The concept of strong smoothness (essentially a second order upper
bound on a function) has also been in play in a different literature,
for the analysis of the concentration of martingales in \emph{smooth}
Banach spaces~\citep{Pinelis94,Pisier75}. This body of work seeks to
understand the concentration properties of a random variable
$\|X_t\|$, where $X_t$ is a (vector valued) martingale and $\|\cdot\|$
is a smooth norm, say an $L_p$-norm.

Recently, \cite{JuditskyNe08} used the fact that a \emph{norm} is strongly convex
if and only if its conjugate is strongly smooth. This duality was
useful in deriving concentration properties of a random variable $\|\M\|$,
where now $\M$ is a random matrix. The norms considered here were
the (Schatten) $L_p$-matrix norms and certain ``block'' composed norms (such as the
$\|\cdot\|_{2,q}$ norm).

\subsection{Organization}

The rest of the paper is organized as follows. In
\secref{sec:preliminaries}, we describe the general family of learning
algorithms. In particular, after presenting the duality of
strong-convexity/strong-smoothness, we isolate an important inequality
(\corref{cor:important}) and show that this inequality alone
seamlessly yields regret bounds in the online learning model and
Rademacher bounds (that leads to generalization bounds in the batch
learning model).  We further highlight the importance of strong
convexity to matrix learning applications by drawing attention to
families of strongly convex functions over matrices. To do so, we rely
on the recent results of \cite{JuditskyNe08}. In particular, we obtain
a strongly convex function over matrices based on strongly convex
vector functions, which leads to a number of corollaries relevant to
problems of recent interest. Next, in \secref{sec:reg} we show how the
obtained bounds can be used for systematically choosing an adequate
prior knowledge (i.e. regularization) based on properties of the given
task.  We then turn to describe the applicability of our approach to
more complex prediction problems. In particular, we study multi-task
learning (\secref{sec:multi-task}), multi-class categorization
(\secref{sec:multi-class}), and kernel learning
(\secref{sec:kernels}). Naturally, many of the algorithms we derive
have been proposed before. Nevertheless, our unified analysis enables
us to simplify previous analyzes, understand the merits and pitfalls
of different schemes, and even derive new algorithms/analyses.

%% file: regularize.bbl
\begin{thebibliography}{37}
\providecommand{\natexlab}[1]{#1}
\providecommand{\url}[1]{\texttt{#1}}
\expandafter\ifx\csname urlstyle\endcsname\relax
  \providecommand{\doi}[1]{doi: #1}\else
  \providecommand{\doi}{doi: \begingroup \urlstyle{rm}\Url}\fi

\bibitem[Agarwal et~al.(2008)Agarwal, Rakhlin, and Bartlett]{AgarwalRaBa08}
Alekh Agarwal, Alexander Rakhlin, and Peter Bartlett.
\newblock Matrix regularization techniques for online multitask learning.
\newblock Technical report, EECS Department, University of California,
  Berkeley, 2008.

\bibitem[Amit et~al.(2007)Amit, Fink, Srebro, and Ullman]{AmitFiSrUl07}
Yonatan Amit, Michael Fink, Nathan Srebro, and Shimon Ullman.
\newblock Uncovering shared structures in multiclass classification.
\newblock In \emph{Proceedings of the 24th International Conference on Machine
  Learning}, 2007.

\bibitem[Bach(2008)]{Bach08}
Francis Bach.
\newblock Consistency of the group lasso and multiple kernel learning.
\newblock \emph{JMLR}, 9, 2008.

\bibitem[Ball et~al.(1994)Ball, Carlen, and Lieb]{BallCaLi94}
Keith Ball, Eric~A. Carlen, and Elliott~H. Lieb.
\newblock Sharp uniform convexity and smoothness inequalities for trace norms.
\newblock \emph{Invent. Math.}, 115:\penalty0 463--482, 1994.

\bibitem[Bartlett and Mendelson(2002)]{BartlettMe02}
P.~L. Bartlett and S.~Mendelson.
\newblock Rademacher and {G}aussian complexities: {R}isk bounds and structural
  results.
\newblock \emph{Journal of Machine Learning Research}, 3:\penalty0 463--482,
  2002.

\bibitem[Bhatia(1997)]{Bhatia97}
R.~Bhatia.
\newblock \emph{Matrix Analysis}.
\newblock Springer, 1997.

\bibitem[Borwein and Lewis(2006)]{BorweinLewis06}
J.~Borwein and A.~Lewis.
\newblock \emph{Convex Analysis and Nonlinear Optimization}.
\newblock Springer, 2006.

\bibitem[Cavallanti et~al.(2008)Cavallanti, Cesa-Bianchi, and
  Gentile]{CavallantiBaGe08}
G.~Cavallanti, N.~Cesa-Bianchi, and C.~Gentile.
\newblock Linear algorithms for online multitask classification.
\newblock In \emph{Proceedings of the Nineteenth Annual Conference on
  Computational Learning Theory}, pages 251--262, 2008.

\bibitem[Cesa-Bianchi and Lugosi(2006)]{CesaBianchiLu06}
N.~Cesa-Bianchi and G.~Lugosi.
\newblock \emph{Prediction, learning, and games}.
\newblock Cambridge University Press, 2006.

\bibitem[Crammer and Singer(2000)]{CrammerSi00}
K.~Crammer and Y.~Singer.
\newblock On the learnability and design of output codes for multiclass
  problems.
\newblock In \emph{Proceedings of the Thirteenth Annual Conference on
  Computational Learning Theory}, 2000.

\bibitem[Gentile(2003)]{Gentile03}
C.~Gentile.
\newblock The robustness of the p-norm algorithms.
\newblock \emph{Machine Learning}, 53\penalty0 (3):\penalty0 265--299, 2003.

\bibitem[Grove et~al.(2001)Grove, Littlestone, and Schuurmans]{GroveLiSc01}
A.~J. Grove, N.~Littlestone, and D.~Schuurmans.
\newblock General convergence results for linear discriminant updates.
\newblock \emph{Machine Learning}, 43\penalty0 (3):\penalty0 173--210, 2001.

\bibitem[Juditsky and Nemirovski(2008)]{JuditskyNe08}
A.~Juditsky and A.~Nemirovski.
\newblock Large deviations of vector-valued martingales in 2-smooth normed
  spaces.
\newblock \emph{submitted to Annals of Probability}, 2008.

\bibitem[Kakade et~al.(2008)Kakade, Sridharan, and Tewari]{KakadeSrTe08}
S.M. Kakade, K.~Sridharan, and A.~Tewari.
\newblock On the complexity of linear prediction: Risk bounds, margin bounds,
  and regularization.
\newblock In \emph{Advances in Neural Information Processing Systems 22}, 2008.

\bibitem[Kivinen and Warmuth(2001)]{KivinenWa01}
J.~Kivinen and M.~Warmuth.
\newblock Relative loss bounds for multidimensional regression problems.
\newblock \emph{Journal of Machine Learning}, 45\penalty0 (3):\penalty0
  301--329, July 2001.

\bibitem[Kivinen and Warmuth(1997)]{KivinenWa97}
J.~Kivinen and M.~Warmuth.
\newblock Exponentiated gradient versus gradient descent for linear predictors.
\newblock \emph{Information and Computation}, 132\penalty0 (1):\penalty0 1--64,
  January 1997.

\bibitem[Lanckriet et~al.(2004)Lanckriet, Cristianini, Bartlett, Ghaoui, and
  Jordan]{LanckrietCrBaElJo04}
G.R.G. Lanckriet, N.~Cristianini, P.L. Bartlett, L.~El Ghaoui, and M.I. Jordan.
\newblock Learning the kernel matrix with semidefinite programming.
\newblock \emph{Journal of Machine Learning Research}, 5:\penalty0 27--72,
  2004.

\bibitem[Lewis(1995)]{Lewis95b}
A.~S. Lewis.
\newblock The convex analysis of unitarily invariant matrix functions.
\newblock \emph{Journal of Convex Analysis}, 2\penalty0 (2):\penalty0 173--183,
  1995.

\bibitem[Littlestone(1988)]{Littlestone88}
N.~Littlestone.
\newblock Learning quickly when irrelevant attributes abound: A new
  linear-threshold algorithm.
\newblock \emph{Machine Learning}, 2:\penalty0 285--318, 1988.

\bibitem[Lounici et~al.(2009)Lounici, Pontil, Tsybakov, and van~de
  Geer]{Lounici09}
Karim Lounici, Massimiliano Pontil, Alexandre~B Tsybakov, and Sara van~de Geer.
\newblock Taking advantage of sparsity in multi-task learning.
\newblock \emph{arXiv:0903.1468}, Mar 2009.

\bibitem[Maurer(2006)]{Maurer06}
Andreas Maurer.
\newblock Bounds for linear multi-task learning.
\newblock \emph{Journal of Machine Learning Research}, 2006.

\bibitem[Meir and Zhang(2003)]{MeirZh03}
R.~Meir and T.~Zhang.
\newblock Generalization error bounds for {B}ayesian mixture algorithms.
\newblock \emph{Journal of Machine Learning Research}, 4:\penalty0 839--860,
  2003.

\bibitem[Ng(2004)]{Ng04}
A.Y. Ng.
\newblock Feature selection, $l_1$ vs. $l_2$ regularization, and rotational
  invariance.
\newblock In \emph{Proceedings of the Twenty-First International Conference on
  Machine Learning}, 2004.

\bibitem[Obozinski et~al.(2007)Obozinski, Taskar, and Jordan]{ObozinskiTaJo07}
G.~Obozinski, B.~Taskar, and M~Jordan.
\newblock Joint covariate selection for grouped classification.
\newblock Technical Report 743, Dept. of Statistics, University of California
  Berkeley, 2007.

\bibitem[Pinelis(1994)]{Pinelis94}
I.~Pinelis.
\newblock Optimum bounds for the distributions of martingales in banach spaces.
\newblock \emph{Ann. Probab}, 22\penalty0 (4):\penalty0 1679--1706, 1994.

\bibitem[Pisier(1975)]{Pisier75}
G.~Pisier.
\newblock Martingales with values in uniformly convex spaces.
\newblock \emph{Israel Journal of Mathematics}, 20\penalty0 (3--4):\penalty0
  326--350, 1975.

\bibitem[Rockafellar(1970)]{Rockafellar70}
R.T. Rockafellar.
\newblock \emph{Convex Analysis}.
\newblock Princeton University Press, 1970.

\bibitem[Shalev-Shwartz(2007)]{Shalev07}
S.~Shalev-Shwartz.
\newblock \emph{Online Learning: Theory, Algorithms, and Applications}.
\newblock PhD thesis, The Hebrew University, 2007.

\bibitem[Shalev-Shwartz and Singer(2006)]{ShalevSi06c}
S.~Shalev-Shwartz and Y.~Singer.
\newblock Convex repeated games and {F}enchel duality.
\newblock In \emph{Advances in Neural Information Processing Systems 20}, 2006.

\bibitem[Shalev-Shwartz and Singer(2007)]{ShalevSi07MLJ}
S.~Shalev-Shwartz and Y.~Singer.
\newblock A primal-dual perspective of online learning algorithms.
\newblock \emph{Machine Learning Journal}, 2007.

\bibitem[Shalev-Shwartz and Singer(2008)]{ShalevSi08}
S.~Shalev-Shwartz and Y.~Singer.
\newblock On the equivalence of weak learnability and linear separability: New
  relaxations and efficient boosting algorithms.
\newblock In \emph{Proceedings of the Nineteenth Annual Conference on
  Computational Learning Theory}, 2008.

\bibitem[Srebro and Ben-David(2006)]{SrebroBe06}
N.~Srebro and S.~Ben-David.
\newblock Learning bounds for support vector machines with learned kernels.
\newblock In \emph{Proceedings of the Nineteenth Annual Conference on
  Computational Learning Theory}, pages 169--183, 2006.

\bibitem[Warmuth and Kuzmin(2006)]{WarmuthKu06}
M.~Warmuth and D.~Kuzmin.
\newblock Online variance minimization.
\newblock In \emph{Proceedings of the Nineteenth Annual Conference on
  Computational Learning Theory}, 2006.

\bibitem[Yang et~al.(2009)Yang, Kim, and Xing]{YangKiXi09}
X.~Yang, S.~Kim, and E.~P. Xing.
\newblock Heterogeneous multitask learning with joint sparsity constraints.
\newblock In \emph{Advances in Neural Information Processing Systems 23}, 2009.

\bibitem[Ying and Campbell(2009)]{YingCa09}
Y.~Ying and C.~Campbell.
\newblock Generalization bounds for learning the kernel.
\newblock In \emph{COLT}, 2009.

\bibitem[Yuan and Lin(2006)]{YuanLi06}
M.~Yuan and Y.~Lin.
\newblock Model selection and estimation in regression with grouped variables.
\newblock \emph{Journal of the Royal Statistical Society: Series {B}},
  68\penalty0 (1):\penalty0 49--67, 2006.

\bibitem[Zalinescu(2002)]{Zalinescu02}
C.~Zalinescu.
\newblock \emph{Convex analysis in general vector spaces}.
\newblock World Scientific Publishing Co. Inc., River Edge, NJ, 2002.

\end{thebibliography}
